\documentclass[lettersize,journal]{IEEEtran}
\usepackage{amsmath,amsfonts}
\usepackage{algorithm}
\usepackage{array}
\usepackage[caption=false,font=normalsize,labelfont=sf,textfont=sf]{subfig}
\usepackage{textcomp}
\usepackage{stfloats}
\usepackage{url}
\usepackage{verbatim}
\usepackage{graphicx}
\usepackage{cite}

\usepackage{cleveref}
\usepackage{algorithmicx}
\usepackage{amssymb}
\usepackage{algpseudocode}
\usepackage{newfloat}
\usepackage{listings}
\usepackage{amsthm}

\usepackage{booktabs}   % For \toprule, \midrule, \bottomrule
\usepackage{amsmath}    % For mathematical symbols and environments
\usepackage{amssymb}    % For additional math symbols

\usepackage{color}

\theoremstyle{definition}
\newtheorem{definition}{Definition}

\newtheorem{theorem}{Theorem}
\newtheorem*{remark}{Remark}

\begin{document}

\title{Control the GNN: Utilizing Neural Controller with Lyapunov Stability for Test-Time Feature Reconstruction}

\author{Jielong Yang, Rui Ding, Feng ji, Hongbin Wang, and Linbo Xie}

%         <-this % stops a space

\markboth{Journal of \LaTeX\ Class Files,~Vol.~14, No.~8, August~2021}%
{Shell \MakeLowercase{\textit{et al.}}: A Sample Article Using IEEEtran.cls for IEEE Journals}


\maketitle

\begin{abstract}
  The performance of graph neural networks (GNNs) is susceptible to discrepancies between training and testing sample distributions. Prior studies have attempted to mitigating the impact of distribution shift by reconstructing node features during the testing phase without modifying the model parameters. However, these approaches lack theoretical analysis of the proximity between predictions and ground truth at test time. In this paper, we propose a novel node feature reconstruction method grounded in Lyapunov stability theory. Specifically, we model the GNN as a control system during the testing phase, considering node features as control variables. A neural controller that adheres to the Lyapunov stability criterion is then employed to reconstruct these node features, ensuring that the predictions progressively approach the ground truth at test time. We validate the effectiveness of our approach through extensive experiments across multiple datasets, demonstrating significant performance improvements.
\end{abstract}

\begin{IEEEkeywords}
graph neural network, Lyapunov stability, neural controller
\end{IEEEkeywords}

\section{Introduction}
\IEEEPARstart{G}raph neural networks (GNNs) have demonstrated significant success in various applications due to their powerful representational capabilities \cite{wu2020comprehensive, yuan2022explainability}. However, their performance suffers significantly when distribution discrepancies exist between training and testing samples \cite{qin2022graph, li2022out}. Existing approaches often tackle distribution shifts during the training phase through data augmentation \cite{sui2024unleashing, ding2022data}, causal inference\cite{fan2023generalizing}, and regularization\cite{buffelli2022sizeshiftreg}. Nevertheless, these strategies necessitate modifications to the model structure and relearning the parameters, limiting their applicability when retraining is costly or impossible.

To address this limitation, recent methods propose reconstructing node features during the testing phase, aiming to mitigate distribution shifts without altering the original model parameters. Notably, the term "node features" here refers explicitly to the original input features of the nodes. For instance, GTrans~\cite{GTrans} generates perturbed node features guided by a surrogate loss to minimize distribution shift effects, but these reconstructed features often lack practical feasibility due to arbitrary perturbations. FRGNN~\cite{FRGNN} addresses this by employing a multi-layer perceptron (MLP) to learn class-representative embeddings, substituting the original node features to enhance robustness. However, both methods rely heavily on heuristic or manually designed mechanisms, lacking theoretical guarantees regarding the proximity of reconstructed features to their ideal representations. Consequently, the interpretability and reliability of these feature reconstruction methods remain limited.

To overcome these limitations, we propose to view the node feature reconstruction problem through the lens of control theory. Specifically, we formalize GNN predictions as system states and node features as controllable inputs. Under this framework, reconstructing node features becomes a control problem, where the goal is to adjust node features precisely to minimize the discrepancy between the GNN predictions and the ground truth. Compared to heuristic optimization approaches, this perspective offers distinct advantages:
\begin{itemize}
    \item It enables precise control over individual node states, allowing fine-grained optimization of prediction accuracy.
    \item It facilitates stability analysis through established stability criteria from control theory, such as Lyapunov stability, ensuring that feature adjustments do not introduce instability or degrade prediction performance.
\end{itemize}

However, applying traditional control methods, such as proportional-integral-derivative (PID) controllers~\cite{1453566} or linear quadratic regulators (LQR)~\cite{10.5555/578807}, faces significant challenges in GNNs. Due to the complex nonlinearity and inherent message-passing mechanisms, adjusting a node's feature impacts not only its own prediction but also influences other nodes' predictions through the graph structure. This complex interdependency renders traditional control methods inadequate.

To address this challenge, we employ a data-driven approach using neural controllers, capable of learning complex and nonlinear control policies directly from data. Neural controllers have demonstrated effectiveness in controlling complex nonlinear systems, such as inverted pendulums and unmanned aerial vehicles~\cite{8676108, Kim1994NonlinearFC, gu2020uav}. However, guaranteeing the stability of neural controllers in highly interconnected systems like GNNs is non-trivial. Manually designing Lyapunov functions for these nonlinear systems is typically difficult or even infeasible. Therefore, we introduce a neural Lyapunov function to simultaneously learn with the neural controller in a data-driven manner. By co-optimizing both the neural controller and the neural Lyapunov function, we ensure that the controller adheres to Lyapunov stability criteria, guaranteeing stable and reliable adjustments of node features during testing.

Although neural Lyapunov methods have shown success in traditional control tasks~\cite{DBLP:journals/corr/abs-2005-00611, zhou2022neural, DBLP:journals/corr/abs-2109-14152}, their application to GNN node feature reconstruction remains unexplored. Our approach extends neural control methods into the domain of graph neural networks, providing a novel solution for theoretically-grounded node feature reconstruction. 

\begin{itemize}
    \item We present a novel control-theoretic paradigm to enhance GNN robustness against distribution shifts without altering model parameters. By formalizing node predictions as system states and node features as controlled variables, we introduce theoretically stability analysis into test-phase feature reconstruction.
    \item We design a neural controller specifically tailored to the nonlinear, interconnected structure of GNNs, capable of effectively reconstructing node features at test time. To guarantee stability during this reconstruction process, we further propose a neural Lyapunov function learned simultaneously with the neural controller, providing theoretical stability assurances for our method.
    \item We empirically validate our proposed framework on multiple benchmark datasets. The experimental results demonstrate that our method significantly improves GNN performance under distribution shift scenarios, effectively reconstructing meaningful node features and maintaining prediction stability without modifying the original model structure or retraining its parameters.
\end{itemize}

\section{Related Work}

\subsection{Test-Time Feature Reconstruction For GNNs}

In the task of semi-supervised node classification, GNNs often suffer from a distribution shift between the training set and testing set due to the inappropriate selection of training samples or limited training data. To tackle this issue without altering the model structure and parameters, reconstructing node features during the testing phase is effective. GTrans~\cite{GTrans} employs surrogate loss to guide the update of perturbations during the testing phase to reconstruct node features. FRGNN~\cite{FRGNN} proves that replacing the node features of labeled nodes with class representative embeddings during the testing phase enhances the model's robustness to distribution shift. Identifying appropriate class representative embeddings is crucial for this method. FRGNN maps the relationship between GNN outputs and inputs using an MLP. Then, it uses the one-hot embedding of the class as input to the MLP to obtain the class representative embeddings. However, the class representative embeddings obtained in this way do not guarantee that the new output of GNN will be close to the one-hot embedding of the corresponding class, which fails to meet the requirements of class representative embeddings. To address this issue, we design a neural controller to find appropriate class representative embeddings by ensuring that the new prediction of GNN gets close enough to the ground truth.

\begin{figure*}[ht]
    \centering
    \includegraphics[width=1\linewidth]{lypunov_gnn}
    \caption{The overall process of our method. We input the original prediction $\hat{Y}t$ into the neural controller to obtain adjusted node features, and then feed these adjusted node features back into the well-trained GNN to obtain the prediction $\hat{Y}{t+1}$ at time $t+1$. Simultaneously, we input the original prediction $\hat{Y}_t$ into the neural Lyapunov function to obtain the Lyapunov value. We then use these values to compute the Lyapunov loss and update the parameters of both the neural controller and the neural Lyapunov function. Afterward, we use a solver to check whether there are states in the state space that do not satisfy the Lyapunov conditions. If such states exist, we add them to the training set and continue training until none remain.}
    \label{fig:1}
\end{figure*}

\subsection{Neural Controller and Neural Lyapunov Function}

Lyapunov stability is commonly used in control theory to ensure system stability. However, designing Lyapunov functions of neural controllers is challenging. To address this issue, ~\cite{DBLP:journals/corr/abs-2005-00611} presents a learning framework to learn control policies and their corresponding Lyapunov functions. This framework guides model training by continuously searching for counterexamples that do not satisfy the Lyapunov condition until all states in the domain satisfy the Lyapunov condition. However, this method requires the Lyapunov function to be differentiable for counterexample search with the satisfiability modulo theories (SMT) solvers. To meet this requirement, all activation functions in ~\cite{DBLP:journals/corr/abs-2005-00611} are Tanh. To relax the restriction that the activation function must be continuous, ~\cite{DBLP:journals/corr/abs-2109-14152} proposes using Mix Integer Programs (MIP) to find counterexamples. This method allows the widely used ReLU activation function to be used. The authors rigorously prove the validity of the Lyapunov function. ~\cite{zhou2022neural} learns the unknown nonlinear system through a neural network without the need for precise system modeling. Then, the system is controlled using the framework proposed in ~\cite{DBLP:journals/corr/abs-2005-00611}. The authors theoretically guarantee the validity of the Lyapunov function under the unknown system. Although these methods simplify the design of Lyapunov functions, they solve traditional control problems rather than enhancing the neural networks' performance with Lyapunov functions. Our work is inspired by these studies. We design a neural controller to reconstruct node features to improve model prediction.

\begin{table}[]
    \centering
    \caption{Summary of variables and their meanings.}
    \begin{tabular}{lp{7cm}}
    \toprule
    \textbf{Symbol} & \textbf{Definition / Meaning} \\
    \midrule
    $h_{c_i}^*$ & Class representative embedding of node $i$ with label $c_i$, satisfying $\|C(h_{c_i}^*) - c_i\|_2 \leq \epsilon$ \\
    $C$ & Trainable classifier in GNNs (e.g., MLP) \\
    $\epsilon$ & A small positive number used as a threshold and to ensure numerical stability \\
    $c_i$ & Class label of node $i$ \\
    $\hat{Y}$ & Node predictions, $\hat{Y} = \text{GNN}(X, A)$ \\
    $X$ & Node features \\
    $A$ & Adjacency matrix of the graph \\
    GNN & Graph Neural Network model, $\text{GNN}(X, A) = C(g(X, A))$ \\
    $g$ & Aggregation function in GNNs (e.g., message-passing function) \\
    $f$ & Function to reconstruct node features based on existing predictions $\hat{Y}$ \\
    $Y$ & Ground truth labels \\
    $\hat{Y}_t$ & Predictions at time $t$ \\
    $\hat{Y}_{t+1}$ & Predictions at time $t+1$ \\
    $f_{\theta}$ & Neural controller parameterized by $\theta$ \\
    $V$ & Lyapunov function used to assess system stability \\
    $V_{\phi}$ & Neural Lyapunov function parameterized by $\phi$ \\
    $\Delta V$ & Difference of the Lyapunov function between time steps, $\Delta V = V(C_{f_{\theta}}(\hat{Y}_t)) - V(\hat{Y}_t)$ \\
    $\Delta V_{\phi}$ & Difference of the neural Lyapunov function, $\Delta V_{\phi}(\hat{Y}_t) = V_{\phi}(C_{f_{\theta}}(\hat{Y}_t)) - V_{\phi}(\hat{Y}_t)$ \\
    $\mathcal{D}$ & Domain of the control system \\
    $\sigma$ & ReLU activation function \\
    $N$ & Number of training samples \\
    $\mathcal{L}_{\phi, \theta}$ & Lyapunov loss function, $\mathcal{L}_{\phi, \theta} = \dfrac{1}{N}\sum\limits_{i=1}^{N} \left( \sigma(-V_{\phi}(\hat{Y}_t^i)) + \sigma(\Delta V_{\phi}(\hat{Y}_t^i)) \right) + V_{\phi}^2(Y)$ \\
    $r$ & Replacement function used in test-time feature reconstruction \\
    $\mathcal{C}$ & Set of states violating Lyapunov conditions, $\mathcal{C} = \left\{ \hat{Y} \ \bigg| \ \left( V_{\phi}(\hat{Y}) \leq 0 \ \cup \ \Delta V_{\phi}(\hat{Y}) \geq 0 \right) \cap \| \hat{Y}_i - Y_i \|_2 \geq \epsilon \right\}$ \\
    $\theta$ & Parameters of the neural controller $f_{\theta}$ \\
    $\phi$ & Parameters of the neural Lyapunov function $V_{\phi}$ \\
    $C_{f_{\theta}}$ & Composite function, $C_{f_{\theta}}(\hat{Y}_t) = C(f_{\theta}(\hat{Y}_t))$ \\
    \bottomrule
    \end{tabular}
    \label{tab:variables}
\end{table}

\section{Methodology}
In this section, we introduce the method of reconstructing node features using a neural controller. The overall process of this method is shown in \Cref{fig:1}. Initially, we use a well-trained GNN to obtain the node predictions. This output is then fed back into the neural controller to reconstruct the node features, ensuring the new predictions remain close to the ground truth. Notably, the term "node features" here refers explicitly to the original input features of the nodes. To ensure that the new predictions derived from the reconstructed node features remain stable near the ground truth, we introduce a neural Lyapunov function to aid in controller training. The variables and symbols used in this method are summarized in \Cref{tab:variables}. This section is organized as follows. We first explain the theoretical foundation of this method. Then, we detail the design of the neural controller and the neural Lyapunov function. Finally, we outline the training approach for this neural controller and the neural Lyapunov function.

\subsection{Problem Formulation}
In FRGNN ~\cite{FRGNN}, the authors demonstrate that replacing the features of labeled nodes with class representative embeddings during the test phase can improve the model's robustness to distribution shift. This paper provides the definition of class representative embeddings, as shown in \Cref{def:1}.

\begin{definition}
    \label{def:1}
    Class Representative Embedding $h_{c_i}^*$: Suppose the embedding $h_{c_i}^*$ of node $i$ with label $c_i$ satisfies the following conditions:
    \begin{equation}
        \label{eq:1}
        ||C(h_{c_i}^*) - c_i||_2 \leq \epsilon,
    \end{equation}
    where $C$ represents the trainable classifier in GNNs and $\epsilon$ denotes a small positive number. Then $h_{c_i}^*$ is defined as the class representative embedding of class $c_i$.
\end{definition}

We draw inspiration from the definition in FRGNN, aiming to identify class representative embeddings with theoretical guarantees. The inference process of the graph neural network is as follows:
\begin{equation}
    \label{eq:2}
    \hat{Y} = \text{GNN}(X, A)
    = C(g(X, A)),
\end{equation}
where $\hat{Y}$ represents the node predictions, $X$ denotes the node features, $A$ is the adjacency matrix of the graph, GNN represents a well-trained graph neural network, $g$ denotes any aggregation function, and $C$ represents the trainable classifier, such as MLP. Our goal is to reconstruct the node features $X$ based on the existing predictions $\hat{Y}$ to ensure that the new predictions get close to the ground truth. This reconstruction process is denoted as $f$. Formally, this can be written as the following optimization problem:
\begin{equation}
    \label{eq:3}
    \min_{f} ||C(f(g(X, A), \hat{Y})) - Y||_2,
\end{equation}
where $f$ is a function we aim to identify to minimize the objective function in \Cref{eq:3}. However, directly solving this optimization problem poses significant challenges. The compositional nature of $f \circ g$ in GNNs introduces complex interactions between node features, graph topology, and the objective. This makes convergence to a global optimum (or even meaningful local optima) theoretically and practically difficult. Directly optimizing node features via gradient-based methods often leads to unstable dynamics, as small perturbations in feature space can propagate through message passing, causing erratic updates. To address this issue, we transform the optimization challenge into a control problem. We can formalize the graph neural network as a control system, treating node features as controlled variables and the predictions as states. However, merely formulating the problem as a control system is not sufficient to guarantee convergence to $Y$ or the stability of the system. This is where the concept of a Lyapunov function becomes essential. By introducing a Lyapunov function into our control design, we can provide a systematic way to analyze and ensure the stability of the system. Therefore, we propose to simultaneously optimize the control problem and learn an appropriate Lyapunov function. This joint optimization allows us to not only find a function $f$ that minimizes the objective in \Cref{eq:3} but also ensures that the system's dynamics are stable and that the predictions converge to the ground truth $Y$. This approach leverages the strengths of control theory to overcome the limitations of directly solving the original optimization problem. The form of this control system is as follows:
\begin{equation}
    \label{eq:4}
    \begin{aligned}
        \hat{Y}_{t+1} = C(f(g(X, A),\hat{Y_t})),
    \end{aligned}
\end{equation}
where $\hat{Y_t}$ represents the predictions at time $t$, and $\hat{Y}_{t+1}$ represents the predictions at time $t+1$. This control system feeds back the predictions at time $t$ to the controller $f$. Then the controller $f$ adjusts the node features to ensure that the subsequent predictions at time $t+1$ get close to the equilibrium point, which is the ground truth. Since GNN is fixed and thus $g(X, A)$ is constant, \Cref{eq:4} can be viewed as a function of the variable $\hat{Y_t}$, as shown below:
\begin{equation}
    \label{eq:5}
    \begin{aligned}
        \hat{Y}_{t+1}= C(f(\hat{Y_t})).
    \end{aligned}
\end{equation}

By transforming the optimization problem in \Cref{eq:3}, we have shifted our focus to designing a controller $f$ and identifying a Lyapunov function. This approach ensures that the system output remains stable near the equilibrium point and that the predictions converge to the ground truth. In our framework, the equilibrium point is defined as the state where the system's output remains unchanged, satisfying $C_{f_{\theta}}(Y)=Y$. When $Y$ is the ground truth, the system $C_{f_{\theta}}(\cdot)$ takes the true label data as input. The model's predicted output will exactly match this real label, resulting in no prediction error that needs further correction. Consequently, the system's next output $C_{f_{\theta}}(Y)$ equals the current input $Y$, thus satisfying the equilibrium condition. By defining equilibrium in this way, we can leverage stability analysis to ensure that the system will drive the predictions toward the ground truth over time. In the following sections, we detail the design of the neural controller and the neural Lyapunov function and introduce the training methodology for these components.

\subsection{Theoretical Foundation}
In the previous section, we transform the GNNs optimization problem into a control problem. Now, the main challenge is designing a controller $f$ to adjust the node features so that the new predictions derived from these reconstructed features get close to the ground truth. Traditional controllers, such as PID, LQR, and AC, are not well-suited for this system due to the complex nonlinearity of graph neural networks. Moreover, unlike ordinary neural networks like MLPs, GNNs include message-passing layers, meaning that modifying a specific node feature not only affects the prediction of that node but also influences the predictions of other nodes through the graph structure. This coupling makes the control problem particularly challenging for traditional controllers. Therefore, we choose a neural controller, leveraging the powerful representation capabilities of neural networks to learn an effective controller suitable for this system. We denote the neural controller as $f_{\theta}$, where $\theta$ represents the trainable parameters. Thus, \Cref{eq:5} can be rewritten in closed-loop form as follows:

\begin{equation}
    \label{eq:6}
    \begin{aligned}
        \hat{Y}_{t+1} = C(f_{\theta}(\hat{Y_t}))\triangleq C_{f_{\theta}}(\hat{Y_t}).
    \end{aligned}
\end{equation}

To ensure the stability of this controller, we employ the Lyapunov stability theory. For a nonlinear closed-loop system, we can use Lyapunov stability criteria to determine system stability. This criterion leverages a scalar function $V(x)$ to assess system stability, known as the Lyapunov function. To apply the Lyapunov stability criterion to the closed-loop system in \Cref{eq:6}, we introduce the following theorem.

\begin{theorem}
    \label{thm1}
    Consider a closed-loop system $\hat{Y}_{t+1} = C_{f_{\theta}}(\hat{Y_t})$, where $C_{f_{\theta}}$ is a MLP with ReLU activation function. Consider the system has an equilibrium point at the ground truth $Y$. Assume there exists a continuously differentiable function $V: \mathbb{R}^n \rightarrow \mathbb{R}$ that satisfies the following conditions: 
    \begin{equation}
        \label{eq:7}
        V(Y) = 0, and, \forall \hat{Y} \in \mathcal{D}\setminus \{Y\}, V(\hat{Y}) > 0 \: and \: \Delta V \leq 0.
    \end{equation}
    where $\mathcal{D}$ is the domain of the control system and $\Delta V$ is defined as:
    \begin{equation}
        \label{eq:8}
        \begin{aligned}
            \Delta V = V(C_{f_{\theta}}(\hat{Y_t})) - V(\hat{Y_t}).
        \end{aligned}
    \end{equation}
    Then, the system is asymptotically stable at $Y$.
\end{theorem}
\begin{proof}
Please refer to the Appendix A.
\end{proof}
\textbf{Remark} (Lyapunov Function and Stability Criterion) Consider a continuously differentiable function $V: \mathbb{R}^n \rightarrow \mathbb{R}$ that maps the state space to a scalar value. For a closed-loop system $\hat{Y}_{t+1} = C_f^\theta(\hat{Y}_t)$ with an equilibrium point at $Y^*$, the system is asymptotically stable at $Y^*$ if $V$ satisfies: (1) $V(Y^*) = 0$, (2) $V(\hat{Y}) > 0$ for all $\hat{Y} \neq Y^*$, and (3) $\Delta V(\hat{Y}) = V(\hat{Y}_{t+1}) - V(\hat{Y}_t) < 0$ for all $\hat{Y} \neq Y^*$. Here $V$ is called a Lyapunov function for the system.

\Cref{thm1} guarantees the stability of the system if the Lyapunov function $V$ satisfies the conditions in \Cref{eq:7}. Our objective is to identify a Lyapunov function $V$ that satisfies \Cref{eq:7}. However, designing a Lyapunov function for nonlinear systems is challenging. To address this issue, we introduce a neural Lyapunov function, which approximates the Lyapunov function through a neural network to derive a candidate Lyapunov function. We denote this learnable candidate neural Lyapunov function as $V_{\phi}$, where $\phi$ represents the trainable parameters. The architecture of the neural network is illustrated in \Cref{fig:2}. The neural Lyapunov function $V_{\phi}$ takes the predictions $\hat{Y_t}$ at time $t$ as input and outputs two components. The first component is the Lyapunov function $V_{\phi}(\hat{Y_t})$, which is used to assess system stability. The second component is the output of the neural controller $f_{\theta}(\hat{Y_t})$, which is used to adjust the node features. 

\begin{figure}
    \centering
    \includegraphics[width=1\linewidth]{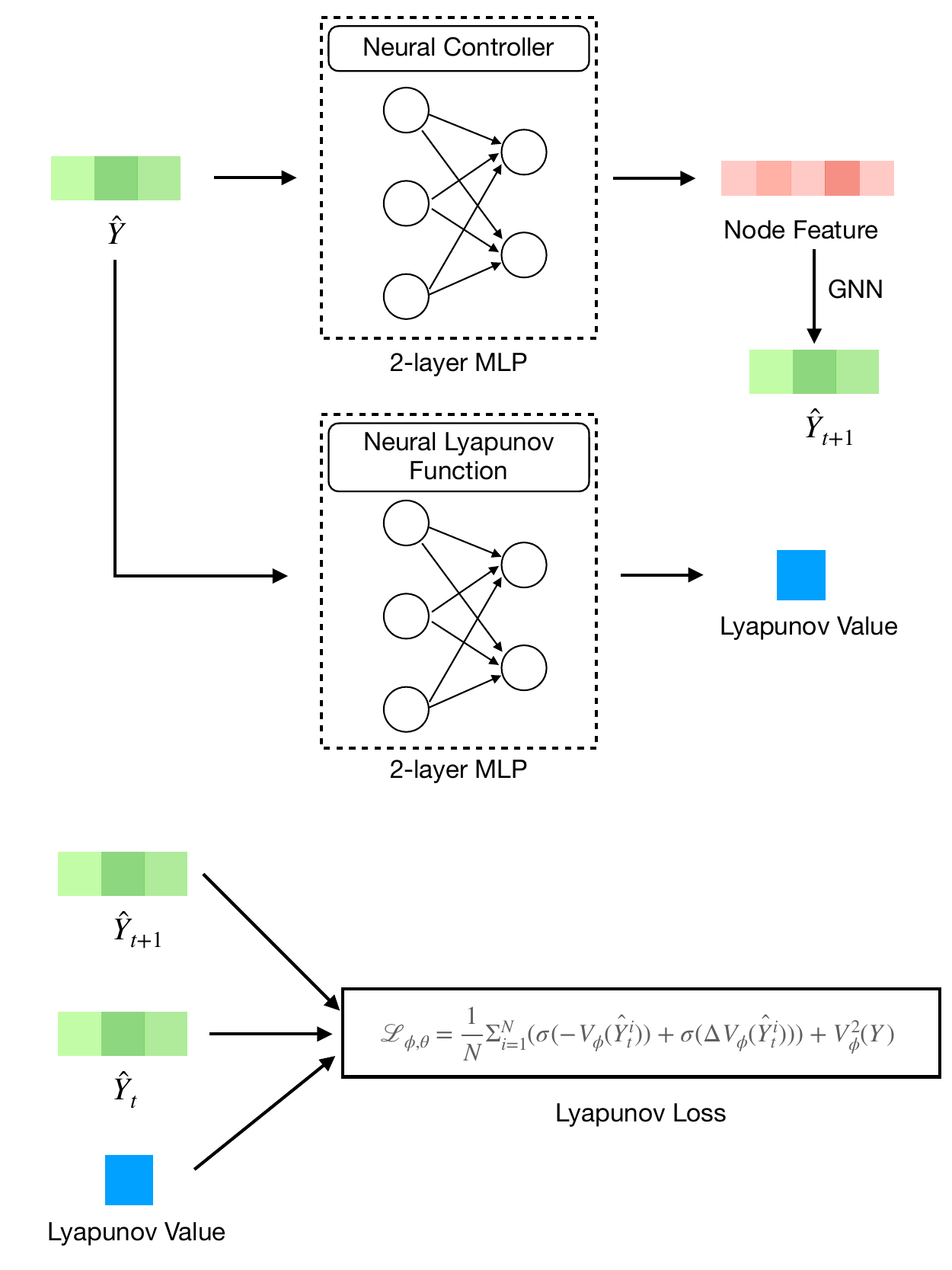}
    \caption{The architecture of the neural controller and the neural Lyapunov function and Lyapunov loss. We simultaneously train the neural Lyapunov function and the neural controller to obtain a control system that ensures stability and effectively guides the GNN predictions to gradually approach the ground truth.}
    \label{fig:2}
\end{figure}

\subsection{Training Methodology}
To obtain the candidate Lyapunov function $V_{\phi}$ for the neural controller $f_{\theta}$, we need to train two neural networks simultaneously, one for the neural controller $f_{\theta}$ and the other for the neural Lyapunov function $V_{\phi}$. Our objective is to find an $f_{\theta}$ that allows the system output to asymptotically stabilize at the equilibrium point (ground truth). The training goal is to ensure that all states $\hat{Y}$ within the domain satisfy the Lyapunov stability criterion. To this end, we design the following Lyapunov loss function:
\begin{equation}
    \label{eq:11}
    \begin{aligned}
 \mathcal{L}_{\phi, \theta} = \frac{1}{N}\Sigma_{i=1}^{N}(\sigma(-V_{\phi}(\hat{Y_t^i})) + \sigma(\Delta V_{\phi}(\hat{Y_t^i}))) + V_{\phi}^2(Y),
    \end{aligned}
\end{equation}
where $N$ represents the number of training samples, $\sigma$ denotes the ReLU activation function, $\sigma(-V_{\phi}(\hat{Y_t^i}))$ ensures that the neural Lyapunov function maps all states except the equilibrium point to positive values, satisfying the Lyapunov stability criterion $\forall x \in \mathcal{D}\setminus \{x_0\}, V(x) > 0$. $\sigma(\Delta V_{\phi}(\hat{Y_i}))$ ensures that the neural Lyapunov function has smaller values closer to the equilibrium point, satisfying the Lyapunov stability criterion $\Delta V \leq 0$. $V_{\phi}^2(Y)$ ensures that the Lyapunov function approaches zero at the equilibrium point, satisfying the Lyapunov stability criterion $V(x_0) = 0$.

The above process defines the training loss function, while we also need to design a method for obtaining training samples. Since the domain of the Lyapunov function is the entire state space, exploring every possible state is infeasible. Therefore, we use an analytical approach to identify counterexamples that violate the Lyapunov criteria within this domain, generating our training samples. \Cref{eq:7} defines the constraints that the Lyapunov function should satisfy. We take the negation of this equation to obtain states that violate the Lyapunov criterion. We define these states as follows:
\begin{equation}
    \label{eq:12}
    \begin{aligned}
        \mathcal{C} = \{\hat{Y} | (V_{\phi}(\hat{Y}) \leq 0 \cup \Delta V_{\phi}(\hat{Y}) \geq 0) \cap \lVert \hat{Y_{i}}-Y_i \rVert _2 \geq \epsilon\},
    \end{aligned}
\end{equation}
where $Y$ is the equilibrium point and $\epsilon$ is a small positive number. We use the SMT solver to find the counterexamples that do not satisfy $\cal C$. The condition $V_{\phi}(\hat{Y}) \leq 0 \cup \Delta V_{\phi}(\hat{Y}) \geq 0$ restricts the found state $\hat{y}$ from satisfying the conditions in \Cref{eq:7}. The constraint $\lVert \hat{Y_{i}}-Y_i \rVert _2 \geq \epsilon$ ensures numerical stability by excluding states too close to equilibrium, which may otherwise introduce numerical instability  during the SMT solving process. It is important to emphasize that these counterexamples are not limited to predictions generated from the original well-trained GNN. Instead, we employ an SMT solver to systematically explore a broader space of node prediction states, intentionally enrich the training set for training the neural controller $f_{\theta}$ and the neural Lyapunov function $V_{\phi}$. Our training process is detailed in \Cref{alg:1}.

\begin{algorithm}[ht]
    \caption{Counterexample-Guided Neural Lyapunov Function Training}
    \label{alg:1}
    \begin{algorithmic}[1]
    \Function{Training}{$\hat{Y}, Y_i$}
        \State $\Delta V_{\phi} = V_{\phi}(C(f_{\theta}(\hat{Y}))) - V_{\phi}(\hat{Y})$
        \State $\mathcal{L}_{\phi, \theta} = \frac{1}{N}\Sigma_{i=1}^{N}(\sigma(-V_{\phi}(\hat{Y_i})) + \sigma(\Delta V_{\phi}(\hat{Y_i})))+V_{\phi}^2(Y_i)$
        \State $\theta \leftarrow \theta - \alpha \nabla_{\theta}\mathcal{L}$
        \State $\phi \leftarrow \phi - \alpha \nabla_{\phi}\mathcal{L}$\\
        \Return $V_{\phi}, f_{\theta}$
    \EndFunction\\
    \Function{Counter Generation}{$V_{\phi}, f_{\theta}$}
        \State $\mathcal{C} = \{y | (V_{\phi}(y) \leq 0 \cup \Delta V_{\phi}(y) \geq 0) \cap \lVert y-Y_i \rVert _2 \geq \epsilon\}$
        \State Use SMT solver to find $y$ that do not satisfy $\mathcal{C}$\\
        \Return Counterexamples $y$
    \EndFunction\\

    \Function{Main}{}
        \State Prepare closed-loop system $C$ and $\hat{Y}$ derived from well-trained GNN
        \State Initialize $V_{\theta}$ and initialize $f_{\theta}$ randomly
        \While{$y$ is not empty}
            \State $V_{\phi}, f_{\theta} \leftarrow$ \Call{Training}{$\hat{Y}, Y_i$}
            \State $y \leftarrow$ \Call{Counter Generation}{$V_{\phi}, f_{\theta}$}
            \State $\hat{Y} \leftarrow \hat{Y} \cup y$
        \EndWhile
    \EndFunction
    \end{algorithmic}
\end{algorithm}

\subsection{Test-Time Feature Reconstruction}
After training the neural controller $f_{\theta}$ and the neural Lyapunov function $V_{\phi}$, we can use the neural controller to adjust the node features during the testing phase. The inference formula is as follows:
\begin{equation}
    \label{eq:13}
    \begin{aligned}
        h_{c_i}^* &= f_{\theta}(Y),\\
        \hat{Y} &= \text{GNN}(r(h_{c_i}^*, X), A),
    \end{aligned}
\end{equation}
where $h_{c_i}^*$ represents the class representative embedding of class $c_i$ and $r$ denotes the replacement function. The neural controller $f_{\theta}$ adjusts the node features to obtain the class representative embeddings $h_{c_i}^*$ and then we replace the features of labeled nodes with $h_{c_i}^*$. This process enhances the model's robustness to distribution shift.

It is important to note that we do not need to perform multiple iterations during testing. The reason is that the trained neural controller and neural Lyapunov function have already stabilized the control system. When the system is at the equilibrium point, feeding this state back into the controller results in control inputs that do not alter the equilibrium point. Therefore, a single adjustment suffices to bring the node features to stable class representative embeddings, and further iterations are unnecessary.

\begin{table*}[htp]
    \centering
    \caption{Results(\%) for the semi-supervised node classification task on the four datasets. Training split represents the method for selecting training nodes. 'Bias' means that training nodes are selected using the Scalable Biased Sampler~\cite{zhu2021shift}. Biased training samples simulate the situation with distribution shift. Each result is reported as an average $\pm$ standard deviation across 10 experiments. The footnote of FRGNN indicates the base model of FRGNN. $*$ indicates that the FRGNN only replaces one class of node features. OOM indicates that the method runs out of memory. The best results are highlighted in bold and the second-best results are underlined.}
    \begin{tabular}{lllllll}
    \hline
    Method(Year) & Training split & Cora & Citeseer & Pubmed & Photo & Computers\\ \hline
    SGC(2019) & Bias & 64.36 $\pm$ 2.90 & 60.46 $\pm$ 2.06 & 50.68 $\pm$ 3.00 & 83.69 $\pm$ 0.37 & 69.35 $\pm$ 0.20\\
    MH-AUG(2021)& Bias & 62.98 $\pm$ 4.05 & 61.40 $\pm$ 1.08 & 51.31 $\pm$ 4.02 & 84.44 $\pm$ 3.71 & 69.74 $\pm$ 2.56\\
    KDGA(2022)& Bias & 66.27 $\pm$ 2.82 & 57.63 $\pm$ 3.01 & \textbf{OOM} & \textbf{OOM} & \textbf{OOM}\\
    $\text{FRGNN}_{\text{SGC}}^*$(2024)& Bias & 64.89 $\pm$ 3.08 & 60.97 $\pm$ 1.89 & 51.27 $\pm$ 3.28 & \underline{84.73 $\pm$ 0.27} & 69.28 $\pm$ 0.34\\
    $\text{FRGNN}_{\text{GCN}}^*$(2024)& Bias & 65.12 $\pm$ 2.89 & 61.02 $\pm$ 2.33 & \underline{52.37 $\pm$ 2.69} & 84.21 $\pm$ 0.35 & 69.59 $\pm$ 0.22\\
    $\text{Ours}_{\text{NC}}$ &Bias & \textbf{67.22 $\pm$ 2.27} & \textbf{62.54 $\pm$ 1.82} & \textbf{52.90 $\pm$ 2.91} & \textbf{84.96 $\pm$ 0.29} & \textbf{70.64 $\pm$ 0.19}\\  \hline
    \label{tab1}
    \end{tabular}
\end{table*}

\section{Experiments}
In this section, we evaluate our method across multiple datasets to demonstrate its effectiveness in reconstructing node features for enhancing GNNs performance. We begin by describing the experimental setup and then present detailed results. Additionally, we visualize predictions before and after reconstruction to assess the impact of our method on GNN predictions. Finally, we analyze the validity of the neural Lyapunov function through visualizations.

\subsection{Experimental Setup}
\subsubsection{Base Model}
We chose SGC ~\cite{pmlr-v97-wu19e} as our base model. SGC is a simple graph neural network that reduces the multi-layer convolution operations of GCN ~\cite{kipf2017semi} to single matrix multiplication, making it more efficient. Our method requires using an SMT solver to identify states that violate Lyapunov constraints, so we must keep the SMT solving time within an acceptable range. Thus, we select SGC as the base model to accelerate the solving process by reducing the graph neural network's parameters.

\subsubsection{Datasets}
We evaluate our method on five popular graph datasets: Cora, Citeseer, Pubmed~\cite{10.5555/3045390.3045396}, Amazon Photo, and Amazon Computers~\cite{shchur2019pitfallsgraphneuralnetwork}. These datasets are standard benchmarks for semi-supervised node classification tasks. The statistics of these datasets are summarized in Appendix B. For each dataset, we follow the setup from FRGNN ~\cite{FRGNN} and apply the Scalable Biased Sampler~\cite{zhu2021shift} to create training, validation, and test sets. This sampling method effectively captures datasets with distribution shifts, allowing us to demonstrate the ability of our method to address such issues.

\begin{table}[]
    \centering
    \caption{Classification accuracy comparison between our neural controller and traditional controller on three datasets. Each result is reported as an average $\pm$ standard deviation across 10 experiments. MFAC indicates that the controller is MFAC and NC indicates that the controller is our neural controller. The best results are highlighted in bold.}
    \begin{tabular}{llll}
    \hline
    Method  & Cora & Citeseer & Pubmed\\ \hline
    SGC & 64.36 $\pm$ 2.90 & 60.46 $\pm$ 2.06 & 50.68 $\pm$ 3.00\\
    $\text{Ours}_{\text{MFAC}}$ & 63.13 $\pm$ 3.42 & 56.73 $\pm$ 3.80 & 51.97 $\pm$ 2.78\\
    $\text{Ours}_{\text{NC}}$ & \textbf{67.22 $\pm$ 2.27} & \textbf{62.54 $\pm$ 1.82} & \textbf{52.90 $\pm$ 2.91}\\  \hline
    \label{tab3}
    \end{tabular}
\end{table}

\subsubsection{Implementation Details}
Our training process for the neural controller and the neural Lyapunov function heavily relies on generating counterexamples, which necessitates ensuring the speed of the SMT solver. Since the speed of SMT solver is proportional to the number of parameters in each graph neural network layer, we downsample the node features to reduce the parameter count. We apply PCA to reduce the feature dimensions to a fixed 20 dimensions, making it the input size for all experiments. We use SGC as the base model, setting its steps of feature propagation to 3 and its output size to the number of classes. Our neural controller $f_{\theta}$ is a two-layer MLP with a hidden size of 16. Our neural Lyapunov function $V_{\phi}$ is also a two-layer MLP with a hidden size of 16 and an output size of 1. We train both $f_{\theta}$ and $V_{\phi}$ using the Adam optimizer with a learning rate of 0.001. We conduct all experiments on an RTX 3090 24G GPU.

During the training phase, we randomly select a node as the initial node and use the one-hot embedding of its class as the equilibrium point $Y$. We use the initial prediction $\hat{Y}$ of this node as the initial state. We then adjust the node features through the neural controller $f_{\theta}$ and input the reconstructed node features into the well-trained GNN to obtain the next prediction $\hat{Y}_{t+1}$. With this information, we calculate the Lyapunov loss and update the neural controller $f_{\theta}$ and the neural Lyapunov function $V_{\phi}$. Then we use the SMT solver to find counterexamples that violate the Lyapunov criterion. When the solver finds such states within the domain, we add these states as counterexamples to the training set. We repeat this process until all states within the domain satisfy the Lyapunov stability criterion. At this point, we obtain a neural controller $f_{\theta}$ that satisfies the Lyapunov stability criterion.

During the testing phase, we input the one-hot embedding of the selected node into the neural controller $f_{\theta}$ to obtain the reconstructed node features. These reconstructed features are then used as class representative embeddings for the selected class, replacing the features of labeled nodes within that class. Finally, we input these updated node features into the well-trained GNN to assess the model's performance on the test set.

\begin{remark}
    In this work, we focus specifically on reconstructing node features belonging to a single node class, treating the ground truth labels of this class as the sole equilibrium point in our Lyapunov stability framework. While extending this method to multiple classes (with multiple equilibrium points) is theoretically straightforward, such an extension significantly increases the complexity of the SMT solving process in Algorithm ~\Cref{alg:1}. Investigating efficient SMT solver to handle multiple equilibrium points requires future research.
\end{remark}

\begin{figure}[htbp]
    \centering
    \includegraphics[width=1\linewidth]{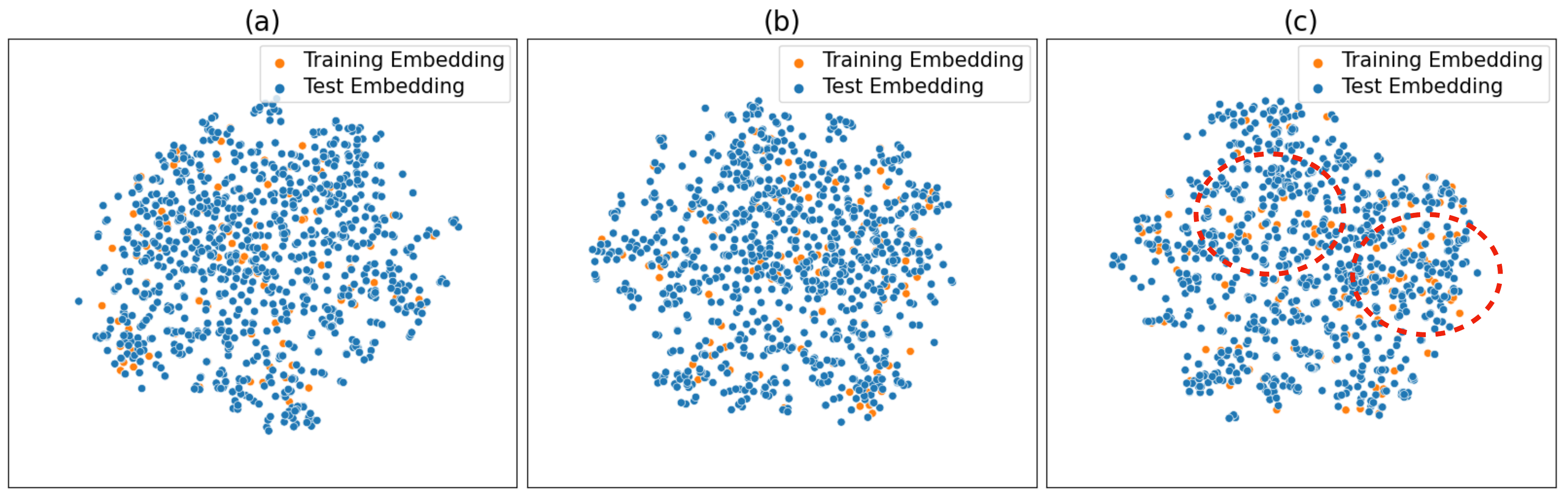}
    \caption{Visualization of the Cora dataset after dimensionality reduction with t-SNE. (a) Node embeddings derived from the original node features; (b) Node embeddings derived from the reconstructed node features by FRGNN; (c) Node embeddings derived from the reconstructed node features by our method. Orange points represent the training node embeddings, while blue points depict the test node embeddings. The red dashed line highlights that our method effectively aligns the embeddings of test nodes more closely with those of the training nodes compared to the original GNN and FRGNN. }
    \label{fig:3}
\end{figure}

\subsubsection{Baseline}
We select MH-AUG ~\cite{c:23}, KDGA ~\cite{c:24}, and FRGNN ~\cite{FRGNN} as the baselines. More detail about these methods is provided in the Appendix C. 

% Notably, our method can only calculate and reconstruct features for one class due to time constraints. To ensure a fair comparison, we only reconstruct node features for a single class in our experiments on FRGNN.
We also compare our neural controller with traditional controllers. PID and LQR are not suitable for this system due to the complex nonlinearity of GNNs. We adopt the MFAC as the controller for test-time feature reconstruction.

\begin{remark}
    Our chosen baselines include both methods that directly reconstruct node features without modifying model parameters or structure (e.g., FRGNN), as well as more sophisticated methods that utilize data augmentation, knowledge distillation, or other techniques (e.g., MH-AUG, KDGA). The motivation for selecting such diverse baselines is to comprehensively illustrate the fundamental differences between our proposed control-theoretic approach and existing mainstream methods designed to mitigate distribution shifts from various perspectives.

    However, due to computational constraints introduced by the SMT solver, our experimental validation is necessarily conducted on relatively small datasets. Like many other works ~\cite{velivckovic2020pointer, li2024another, mustafa2023gats}, our primary objective is not to outperform existing state-of-the-art graph neural networks on standard datasets, but rather to validate the theoretical correctness and practical effectiveness of our novel approach of applying our novel framework to existing GNNs, specifically the application of control theory and neural Lyapunov stability analysis to GNN feature reconstruction problems under distribution shifts.
\end{remark}

\subsection{Results}

\subsubsection{Performance Comparison}

\begin{figure}[htbp]
    \centering
    \includegraphics[width=1\linewidth]{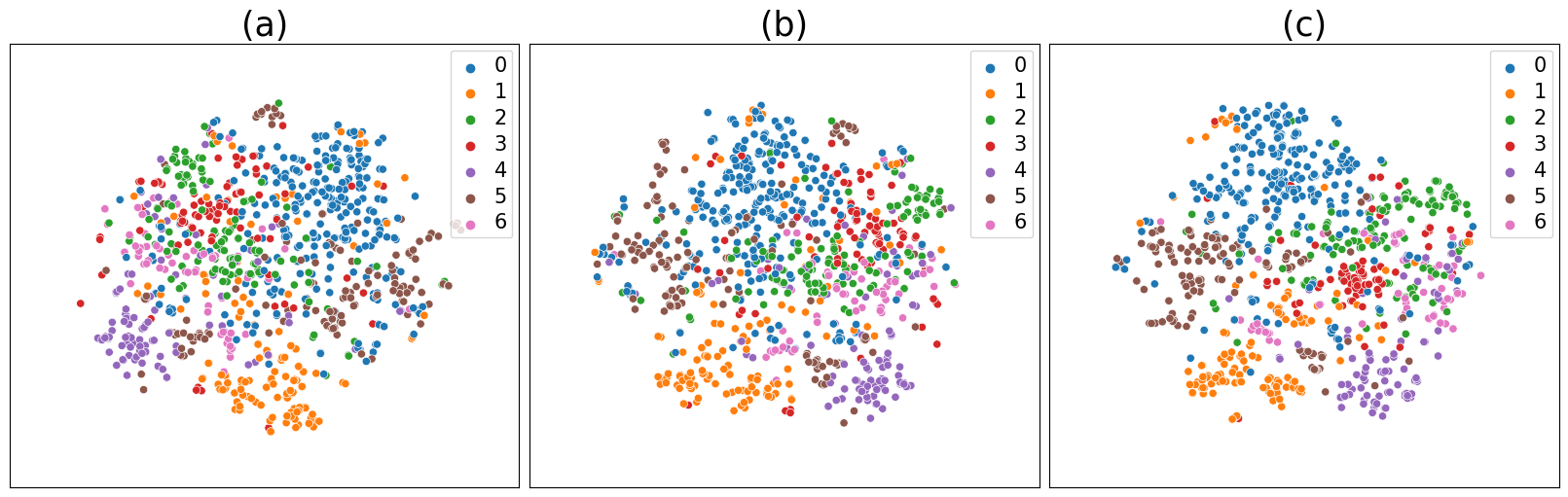}
    \caption{Visualization of the Cora dataset after dimensionality reduction with t-SNE. (a) Node embeddings derived from the original node features; (b) Node embeddings derived from the reconstructed node features by FRGNN; (c) Node embeddings derived from the reconstructed node features by our method. Different colore indicate different categories. Our method shows better clustering of node embeddings for each class than the original GNN and FRGNN.}
    \label{fig:4}
\end{figure}

We evaluated our method on the Cora, Citeseer, Pubmed, Amazon Photo, and Amazon Computers datasets, with results shown in \Cref{tab1}. Our approach outperformed SGC across all datasets, demonstrating its effectiveness and robustness across different statistical characteristics.

Unlike MH-AUG and KDGA, which improve robustness by augmenting data during training, our method enhances robustness by reconstructing features during the testing phase. Our approach consistently achieved superior performance on all datasets compared to MH-AUG and KDGA, indicating the benefits of feature reconstruction at test time.

We further compare our method with FRGNN. FRGNN trains an MLP with GNN predictions and node features to generate class representative embeddings. This paradigm makes FRGNN dependent on the Lipschitz condition, which imposes strict requirements on the MLP's training data. If GNN predictions deviate significantly from the ground truth, it may struggle to find suitable class representative embeddings. Our approach formalizes the graph neural network as a control system and reconstructs node features through a neural controller to ensure the new predictions close to the ground truth. The node features reconstructed by our method can better satisfy the definition of class representative embeddings than FRGNN. Therefore, when replacing the features of a single class of labeled nodes, our method outperforms FRGNN on all datasets.

To further demonstrate the effectiveness of our controller, we compare it with traditional controllers on three datasets, as shown in \Cref{tab3}. For such a control system with complex nonlinearity and coupling, MFAC fails to find the class representative embeddings, leading to a significant performance drop. Our neural controller outperforms MFAC on all datasets, indicating the superiority of neural controllers in adjusting node features for GNNs.

\subsubsection{Visualization}

\begin{figure}[htbp]
    \centering
    \includegraphics[width=1\linewidth]{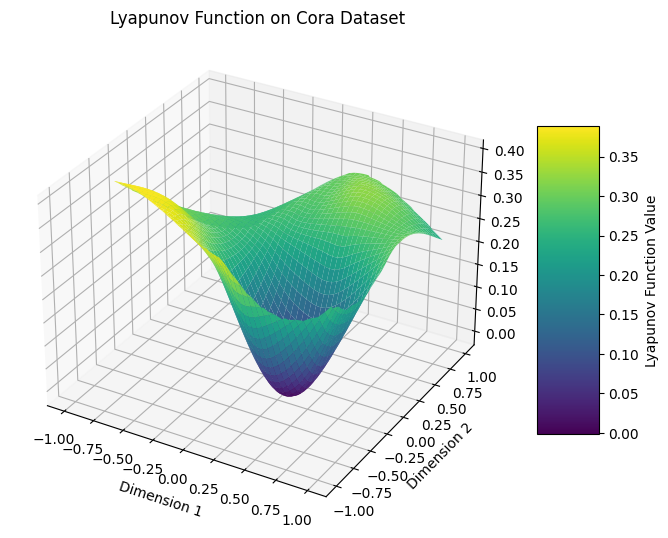}
    \caption{Visualization of Lyapunov function value. Dimension 1 represents dimension 1 of the t-SNE reduced state, Dimension 2 represents dimension 2 of the t-SNE reduced state, and the color represents the value of the Lyapunov function. The Lyapunov function values are zero near the equilibrium point and increase as they move away, indicating that the learned neural Lyapunov function meets the Lyapunov stability criterion.}
    \label{fig:5}
\end{figure}

To illustrate the performance gains in classification from replacing node features, we conduct a t-SNE visualization of node embeddings on the Cora dataset, as shown in \Cref{fig:3}. The figure displays the original node embeddings, those reconstructed by FRGNN, and those reconstructed by our method. In this visualization, blue points represent test nodes and orange points represent training nodes. The results in this figure demonstrate that our method effectively aligns the embeddings of test nodes more closely with those of the training nodes compared to the original GNN and FRGNN scenarios. This alignment reduces the distribution shift between training and test nodes, enhancing the ability of the classifier to accurately categorize embeddings. Therefore our method achieves better classification performance on test nodes. We also visualize the embeddings for each node class in \Cref{fig:4}, where our approach shows better clustering of node embeddings for each class than FRGNN.

To further analyze the Lyapunov function, we visualize its values on the Cora dataset, as shown in \Cref{fig:5}. We sample vectors near the equilibrium point and calculate their corresponding Lyapunov function values. We reduce these vectors and the equilibrium point to two dimensions for visualization with t-SNE. Dimension 1 and Dimension 2 represent the reduced dimensions and the vertical axis shows the Lyapunov function values. The Lyapunov function values are zero near the equilibrium point and increase as they move away, indicating that our Lyapunov function meets the Lyapunov stability criterion. This phenomenon confirms that our controller can maintain the system output near the equilibrium point (i.e., groud truth).

\section{Conclusion}
In this paper, we propose a novel testing-time feature reconstruction method to enhance the robustness of graph neural networks to distribution shifts. We formalize GNNs as a control system, treating node features as controlled variables. By designing a neural controller to reconstruct these node features, we ensure that the predictions gradually stabilize near the ground truth. We introduce a neural Lyapunov function to aid in controller training, ensuring system stability and obtaining class representative embeddings. We evaluate our method across multiple datasets, demonstrating its effectiveness in reconstructing node features and enhancing GNN performance. 

Limitation: currently, the solving time of the solver remains a bottleneck for our method. We have to downsample the node features to reduce the number of parameters in the graph neural network, which may affect the model performance. In the future, we will explore more efficient methods to solve the Lyapunov constraints.

\section*{Appendix A}
\begin*{\textbf{Theorem 1:}}
    \label{thm2}
    Consider a closed-loop system $\hat{Y}_{t+1} = C_{f_{\theta}}(\hat{Y_t})$ with an equilibrium point at the $Y$, i.e. $C_{f_{\theta}}(Y) = Y$, where $C_{f_{\theta}}$ is a MLP with ReLU activation function. Suppose there exists a continuously differentiable function $V: \mathbb{R}^n \rightarrow \mathbb{R}$ that satisfies the following conditions:
    \begin{equation}
        \label{eq:71}
        V(Y) = 0, and, \forall \hat{Y} \in \mathcal{D}\setminus \{Y\}, V(\hat{Y}) > 0 \: and \: \Delta V \leq 0.
    \end{equation}
    where $\mathcal{D}$ is the domain of the control system and $\Delta V$ is defined as:
    \begin{equation}
        \label{eq:81}
        \begin{aligned}
            \Delta V = V(C_{f_{\theta}}(\hat{Y_t})) - V(\hat{Y_t}).
        \end{aligned}
    \end{equation}
    Then, the system is asymptotically stable at the $Y$ and $V$ is called a Lyapunov function.
\end*{theorem}

\begin{proof}
    We consider the neural network controller $f_{\theta}$ and the classifier $C$ as an MLP $C_{f_{\theta}}$ with $l+k$ layers. For an MLP with $l+k$ layers, we have:
    \begin{equation*}
		\begin{aligned}
			C^{l+k}_{f_{\theta}}(x)&=g^{(l+k)}(g^{(l+k-1)}(\cdots g^{(1)}(x)\cdots)),
		\end{aligned}
	\end{equation*}
    where $g^{(i)}$ represents the $i$-th layer of the MLP. For a ReLU activation function, $g^{i}$ can be expressed as:
    \begin{equation*}
        \begin{aligned}
            g^{(i)}(x) = ReLU(W^{(i)}x+b^{(i)}),
        \end{aligned}
    \end{equation*}
    For any $x_1, x_2 \in \mathbb{R}^n$, we have:
    \begin{equation*}
		\begin{aligned}
			&\lVert C^{l+k}_{f_{\theta}}(x_1)-C^{l+k}_{f_{\theta}}(x_2)\rVert _2\\
			&=\lVert g^{(l+k)}(g^{(l+k-1)}(\cdots g^{(1)}(x_1)\cdots))-\\
			&g^{(l+k)}(g^{(l+k-1)}(\cdots g^{(1)}(x_2)\cdots))\rVert _2\\
			&\leq \lVert W^{l+k}g^{(l+k-1)}(\cdots g^{(1)}(x_1)\cdots)+B^{l+k}\\
            &-W^{l+k}g^{(l+k-1)}(\cdots g^{(1)}(x_2)\cdots) - B^{l+k}\rVert _2\\
            &(\text{ReLU Lipschitz Condition})\\
			&\leq \lVert g^{(l+k-1)}(\cdots g^{(1)}(x_1)\cdots)-\\
			&g^{(l+k-1)}(\cdots g^{(1)}(x_2)\cdots)\rVert _2\cdot \lVert W^{(l+k)}\rVert _2\\
            &(\text{Cauchy-Schwarz Inequality})\\
			&\leq \cdots\\
			&\leq \lVert x_1-x_2\rVert _2\cdot \prod_{i=1}^{l+k}\lVert W^{(i)}\rVert _2\\
            &= L\lVert x_1-x_2\rVert _2.
		\end{aligned}
	\end{equation*}
    Therefore, $C_{f_{\theta}}$ is locally Lipschitz continuous. For a locally Lipschitz continuous function, we can apply the Lyapunov stability criterion. Refer to ~\cite{10.5555/578807}, we rewritten the Lyapunov stability criterion as follows:
    \begin{equation*}
        \begin{aligned}
            V(x_0) = 0, and, \forall x \in \mathcal{D}\setminus \{x_0\}, V(x) > 0 \: and \: \Delta V \leq 0,
        \end{aligned}
    \end{equation*}
    where $x_0$ is the equilibrium point, $V$ is the Lyapunov function, $D$ is the domain of the control system, and $\Delta V$ is defined as:
    \begin{equation*}
        \begin{aligned}
            \Delta V = V(F(x)) - V(x),
        \end{aligned}
    \end{equation*}
    where $F$ is locally Lipschitz continuous.We replace $F$ with $C_{f_{\theta}}$ and $x$ with $\hat{Y_t}$, then we have:
    \begin{equation*}
        \begin{aligned}
            V(\hat{Y}) = 0, and, \forall \hat{Y} \in \mathcal{D}\setminus \{Y\}, V(\hat{Y}) > 0 \: and \: \Delta V \leq 0.
        \end{aligned}
    \end{equation*}
    When the above conditions are satisfied, the system is asymptotically stable at the equilibrium point $Y$. This completes the proof.

\end{proof}

\begin{table*}[ht]
    \centering
    \caption{Overall datasets statistics}
    \begin{tabular}{cccccc}
    \hline
    Dataset & Cora & Citeseer & Pubmed & Amazon Photo & Amazon Computers\\ \hline
    \#Nodes & 2,708 & 3,327 & 19,717 & 7487 & 13381\\
    \#Edges & 5,429 & 4,732 & 44,338 & 119043 & 245778\\
    \#Features & 1,433 & 3,703 & 500 & 745 & 767\\
    \#Classes & 7 & 6 & 3 & 8 & 10\\
    \#Label Rate & 5.2\% & 3.6\% & 0.3\% & 2.14\% & 1.49\%\\ \hline
    \#Train Nodes & 140 & 120 & 60 & 160 & 200\\
    \#Val Nodes & 500 & 500 & 500 & 240 & 300\\
    \#Test Nodes & 1,000 & 1,000 & 1,000 & 7250 & 13252\\ \hline
    \label{tab2}
    \end{tabular}
\end{table*}

\section*{Appendix B}

We summarize the statistics of the datasets used in our experiments in \Cref{tab2}. The datasets include Cora, Citeseer, Pubmed, Amazon Photo, and Amazon Computers. Each dataset contains a different number of nodes, edges, features, and classes. The label rate of each dataset is also provided. 

\section*{Appendix C}

MH-AUG improves model generalization through adaptive data augmentation. KDGA enhances the model's adaptability to distribution shift by training a teacher model with augmented data and distilling its knowledge into a student model. FRGNN utilizes the predictions of a well-trained GNN and node features to train an MLP, which generates class representative embeddings using the one-hot embedding of each class. These features are then used to replace the features of labeled nodes, improving robustness against distribution shift.

\bibliographystyle{IEEEtran}
\bibliography{aaai24}

\end{document}